\documentclass[12pt,
a4paper,twoside
]{article}

\usepackage[utf8]{inputenc}
\usepackage[T1]{fontenc}
\usepackage[french,english]{babel}
 
\usepackage{cite}			

\usepackage{lmodern}
\usepackage{dsfont}			
\usepackage{stmaryrd}
\usepackage{bbm}
\usepackage{mathrsfs}
\usepackage{upgreek}

\usepackage[a4paper
,twoside
,hmarginratio=3:4
,vcentering
,headheight=15pt
,textheight=0.7\paperheight,textwidth=0.7\paperwidth
]{geometry}
\usepackage{layout}

\usepackage{relsize}

\usepackage{fancyhdr}
\usepackage{titling}

\usepackage{pstricks}
\usepackage{pifont}
\usepackage{enumerate}
\usepackage{graphicx}
\usepackage[all]{xy}
\usepackage{multicol}
\usepackage{longtable}
\usepackage{amsmath,amsfonts,amssymb,amstext}
\usepackage[amsmath,thmmarks,hyperref,amsthm]{ntheorem}
\usepackage{nicefrac}

\usepackage{etoolbox}
\usepackage{fancyvrb}
\usepackage{marginnote}
\usepackage{cancel}
\usepackage{etex}

\usepackage{tikz}
\usetikzlibrary{trees,arrows}

\usepackage{varioref}
\usepackage[unicode,bookmarks,colorlinks=true,citecolor=black,linkcolor=black
]{hyperref}
\usepackage[all]{hypcap}
\usepackage{cleveref}

\DeclareMathOperator{\supp}{supp}

\newcommand{\ds}{\displaystyle}

\renewcommand{\C}[1]{\mathscr{#1}}
\newcommand{\F}[1]{\mathfrak{#1}}
\newcommand{\B}[1]{\mathds{#1}}
\newcommand{\Up}[1]{\mathrm{#1}}

\newcommand{\wt}[1]{\widetilde{#1}}

\renewcommand{\leq}{\leqslant}

\newcommand{\quotient}[2]{\nicefrac{\mathlarger{#1}}{\mathlarger{#2}}}

\numberwithin{equation}{section}
\pretocmd{\chapter}{
}
{}{}

\theoremstyle{break}
\newtheorem{thm}{Theorem}[section]
\newtheorem{lemma}[thm]{Lemma}

\newtheorem{prop}[thm]{Proposition}

\newtheorem{dfn}{Definition}[section]

\theoremheaderfont{\itshape}

\crefname{dfn}{definition}{definitions}
\Crefname{dfn}{Definition}{Definitions}
\crefname{prop}{proposition}{propositions}
\Crefname{prop}{Proposition}{Propositions}
\crefname{cor}{corollary}{corollaries}
\Crefname{cor}{Corollary}{Corollaries}

\makeatletter
\renewenvironment{proof}[1][\proofname]{
  \th@nonumberplain
  \def\theorem@headerfont{\scshape}%
  \normalfont
  \theoremsymbol{\ensuremath{_\blacksquare}}
  \@thm{proof}{proof}{#1}}%
  {\@endtheorem}
\makeatother

\title{Markov substitute processes : a new model for linguistics and beyond}
\author{Olivier \textsc{Catoni} and Thomas \textsc{Mainguy}}
\date{\today}

\begin{document}

\pagestyle{fancy}
\fancyhf{}
\fancyhead[RE]{\small\sc\nouppercase{\leftmark}}
\fancyhead[LO]{\small\sc\nouppercase{\rightmark}}
\fancyhead[LE,RO]{\thepage}


\maketitle

\begin{abstract}
We introduce Markov substitute processes, a new model at 
the crossroad of statistics and formal grammars, 
and prove its main property : Markov substitute 
processes with a given support form an exponential 
family.
\end{abstract}

\section{Introduction}

We defined in a previous work \cite{Mainguy} Markov substitute models 
with linguistics 
in mind \cite{Stabler09,Roark01}. Our purpose was to propose families 
of probability measures 
on sentences (finite sequences of words 
taken from a given finite dictionary), 
and use them to learn the syntax of a language 
from an i.i.d. random sample of sentences
by performing some suitable kind of statistical 
estimation. 
However, the model we built with this idea in mind 
turned out to be a general purpose extension of 
Markov chains that can be applied to other types of data. 

\section{Definition of Markov substitute models} 

According to our definition, a Markov substitute model 
is a family of probability measures on finite sequences 
of words taken from a finite dictionary $D$. Therefore, the 
state space of the model is $D^+ = \bigcup_{j=1}^{\infty} 
D^j$, the set of all sentences of finite non-zero length. 
Adding the empty string $\varepsilon$ of zero length, 
we will also consider $D^* = \{ \varepsilon \} \cup D^+$.  

In a Markov substitute model, some expressions (subsequences 
of words) can be substituted to others independently 
from the context. To give a precise definition of this 
property, it is useful to introduce the insertion 
operator $\alpha$. It operates on a two sided context 
$x = (x_1, x_2) \in D^* \times D^*$, (where the left and 
right contexts $x_1$ and $x_2$ may be empty strings), 
and an expression $y \in D^+$ ( that is a non-empty 
finite string of words). The insertion operator is defined 
as   
\[ 
\alpha(x, y) = \gamma(x_1, y, x_2), \qquad x \in D^* \times D^*, y \in D^+, 
\] 
where $\gamma$ is the concatenation operator that simply 
pieces strings together. Remark that given $x \in (D^*)^2$, 
the map 
\begin{align*}
D^+ &  \rightarrow D^+ \\ 
y & \mapsto \alpha(x, y) 
\end{align*}
is one to one, whereas for a given $y \in D^+$, 
the map 
\begin{align*}
(D^*)^2 & \rightarrow D^+ \\ 
x & \mapsto \alpha(x, y) 
\end{align*} 
is not, since, for instance, $
\gamma(y,y,y') = \alpha \bigl( 
(y, y'), y \bigr) = \alpha \bigl[ 
\bigl(\varepsilon, \gamma(y,y') \bigr), 
y \bigr]$. 

We define Markov substitute models in terms of Markov 
substitute sets. 
\begin{dfn}
Consider a domain $\C{D} \subset D^+$ (that may be $D^+$ itself
or any subset of it). 
Consider a probability distribution $\Up{P} 
\in \C{M}_+^1(D^+)$ on nonempty strings of words, 
such that $\supp( \Up{P}) \subset \C{D}$.   
Introducing the 
domain $\C{D}$ allows for instance to impose if 
needed that $\Up{P}$ has a finite support.
The support of $\Up{P}$ defines a language
(a subset of non empty strings of words), 
and $\Up{P}$ itself indicates the probability 
to observe any given sentence in this language. 
A set $B \subset D^+$ of syntagms 
is called a Markov substitute set 
of the string distribution $\Up{P}$ 
if and only if 
there exists a skew-symmetric function 
$\beta : B \times B \rightarrow \B{R}$, 
called a substitute exponent, 
such that for any context 
$x \in (D^*)^2$ and any $y, y' \in B$, such that $\{ 
\alpha(x, y), \alpha(x, y') \} \subset \C{D}$, 
\begin{equation}
\label{eq:swapsub}
\Up{P} \bigl( \alpha(x,y') \bigr) = 
\Up{P} \bigl( \alpha(x,y) \bigr) \exp \bigl( \beta (y,y') \bigr). 
\end{equation}
\end{dfn}
(By skew symmetric, we mean that $\beta (y,y') = - \beta (y', y)$.) 
We chose the notation $\beta$ instead of $\beta_{B}$, because in 
the case when $\{y,y'\} \subset B \cap B'$, the intersection of two Markov substitute sets, 
the substitute exponents involved in the property that $B$ is a Markov 
substitute set and in the property that $B'$ is a Markov substitute 
set are the same when the pair $\{y,y'\}$ is active, meaning that 
there is $x \in (D^*)^2$, such that $\Up{P} \bigl( \alpha(x, y) \bigr) > 0$
and $\Up{P} \bigl( \alpha(x, y') \bigr) > 0$ and are arbitrary and therefore 
can also be chosen to be the same when the pair $\{y,y'\}$ is not 
active. 
\begin{prop}
For any Markov substitute set $B$ 
of $\Up{P}$ on the domain $\C{D}$, for any $x_1, x_2 \in (D^*)^2$ and any $y_1, y_2 \in B$ 
such that $\{ \alpha(x_i, y_j), 1 \leq i \leq 2, 1 \leq j \leq 2 \} 
\subset \C{D}$, 
\begin{equation}
\label{eq:swap}
\Up{P} \bigl( \alpha(x_1, y_1) \bigr) 
\Up{P} \bigl( \alpha(x_2, y_2) \bigr) 
= \Up{P} \bigl( \alpha(x_1, y_2) \bigr) 
\Up{P} \bigl( \alpha(x_2, y_1) \bigr).
\end{equation}
\end{prop}
This proposition shows that we can exchange $y_1$ and $y_2$ in 
a pair of independent draws from $\Up{P}$ without changing 
the likelihood of this pair. This generalizes of course to 
larger i.i.d. samples drawn from $\Up{P}$ : we do not change 
the likelihood of the sample if we exchange $y_1$ and $y_2$ 
belonging to the same Markov substitute set $B$.   
\begin{proof}
Using the definition, we see that 
\begin{multline*}
\Up{P} \bigl( \alpha(x_1,y_1) \bigr) \Up{P} \bigl( \alpha(x_2,y_2)
\bigr) \\ = \Up{P} \bigl( \alpha(x_1, y_2) \bigr) \exp \bigl( 
\beta (y_2,y_1) \bigr)
\Up{P} \bigl( \alpha(x_2, y_1) \bigr) \exp \bigl( \beta (y_1, y_2) \bigr) \\ = 
\Up{P} \bigl( \alpha(x_1, y_2) \bigr) 
\Up{P} \bigl( \alpha(x_2, y_1) \bigr). 
\end{multline*}
\end{proof} 
Markov substitute sets have the following elementary properties.
\begin{prop}
\label{prop1.1}
\begin{itemize}
\item A subset of a Markov substitute set is itself a Markov substitute 
set.
\item A set $B \subset D^+$ 
is a Markov substitute set if and only 
if, for any $y, y' \in B$, the pair $\{y, y'\}$ is a Markov substitute 
set.  
\item If $B \subset D^+$ is a Markov substitute set 
and if $x \in (D^*)^2$ 
then 
\[
\alpha(x, B) \overset{\text{\rm def}}{=} \{ \alpha(x, y), y \in B \}
\] is also a Markov substitute set. 
\end{itemize}
\end{prop}
\begin{proof}
The first points are straightforward consequences of \cref{eq:swapsub}. 
As for the last point, writing $x = (x_1, x_2)$, where $x_1, x_2 \in D^*$,   
we see by the definitions that for any $(z_1,z_2) \in (D^*)^2$ 
and any $y, y' \in B$, such that $
\bigl\{ \alpha\bigl(z, \alpha(x, y) \bigr), 
\alpha\bigl(z, \alpha(x, y') \bigr) \bigr\} \subset \C{D}$, 
\begin{multline*}
\Up{P} \bigl[ \alpha \bigl( z, \alpha(x,y') \bigr) \bigr]   
= \Up{P} \bigl[ \alpha \bigl( \bigl( \gamma(z_1, x_1), \gamma(x_2, z_2)
\bigr) , y' \bigr) 
\bigr] \\ = \Up{P} \bigl[ \alpha \bigl( 
\gamma(z_1,x_1), \gamma(x_2, z_2), y) \bigr]  \, \exp \bigl( \beta (y, y') 
\bigr)  
= \Up{P} \bigl[ \alpha \bigl( z, \alpha(x, y) \bigr) \bigr] \, \exp
\bigl( \beta (y, y') \bigr),
\end{multline*}
proving that $\alpha(x, B)$ is a Markov substitute set with substitute 
exponent
\[
\beta \bigl( \alpha(x, y), \alpha(x, y') \bigr) 
= \beta (y,y'), \qquad y, y' \in B. 
\]
\end{proof}
\begin{dfn}
For any given family $\C{B}$ of subsets of $D^+$, we will say that the random 
process $S \in D^+$ is a $\C{B}$-Markov substitute process on the 
domain $\C{D}$ if and only if 
all the members of $\C{B}$ are Markov substitute sets of its probability distribution $\B{P}_{S}$. We will say that the probability distribution 
$\B{P}_S$ of $S$ is a $\C{B}$-Markov 
substitute probability measure on $\C{D}$. We will use the notation 
$\F{M}(\C{D}, \C{B})$ to denote the set of $\C{B}$-Markov substitute 
probability measures on $\C{D}$. 
\end{dfn}

\section{Markov substitute processes as exponential families} 

To describe the possible supports of $\C{B}$-Markov substitute 
processes, we introduce the equivalence relation $\sim_{\C{B}}$ 
on $\C{D}$ that is the smallest one such that, 
for any $x \in (D^*)^2$, $y, y' \in B \in \C{B}$, such 
that $\alpha(x, \{y,y'\}) \subset \C{D}$,  
\[
\alpha(x,y) \sim_{\C{B}} \alpha(x,y'). 
\]
In other words, $\quotient{\C{D}}{\sim_{\C{B}}}$ 
are the connected components of the graph 
\begin{equation}
\label{eq:3.1}
\C{G}(\C{D}, \C{B}) = \bigl\{ \bigl( \alpha(x, y), \alpha(x,y') \bigr), 
x \in (D^*)^2, y, y' \in B \in \C{B} \bigr\} \cap \bigl( 
\C{D} \times \C{D} \bigr).
\end{equation}
It turns out that for any domain $\C{D} \subset D^+$ and any family $\C{B}$ of subsets of 
$D^+$, $\F{M}(\C{D}, \C{B}) \neq \varnothing$. To prove this, 
it will be useful 
to introduce the special family of independent Markov substitute 
processes. Given a strict sub-probability measure $\xi \in \C{M}_+(D)$, 
let us put 
\[ 
r = 1 - \sum_{w \in D} \xi(w) > 0, 
\] 
and let us define the independent process $\wt{S}_{\xi}$ by its 
distribution 
\[ 
\B{P} \bigl( \wt{S}_{\xi} = w_{1:k} \bigr) = \frac{r}{1 - r} \prod_{j=1}^k 
\xi(w_j).
\] 
Let us remark that $\B{P} [ \ell(\wt{S}_{\xi}) = L ] = r(1-r)^{L-1}$. 
It is easy to see from the definition that $\supp(\xi)^+ = \supp (\wt{S}_{\xi}
)$ is 
a Markov substitute set of the independent Markov substitute process 
$\wt{S}_{\xi}$ 
and that its substitute exponent is equal to 
\[
\beta (y, y') = \log \biggl( \frac{\B{P}(\wt{S}_{\xi} = y')}
{\B{P}(\wt{S}_{\xi} = y)} \biggr), \qquad y, y' \in \supp(\xi)^+.
\]
\begin{prop}
\label{prop:2.2}
For any domain $\C{D}$ and any family $\C{B}$ of subsets of $D^+$, for 
any $\C{B}$-Markov substitute measure $\Up{P}$ on $\C{D}$, 
there is a subset $\C{C}_{\Up{P}} \subset \quotient{\C{D}}{\sim_{\C{B}}}$ 
of components of $\C{D}$ 
such that the support $\supp ( \Up{P})$ is of the form
\begin{equation}
\label{eq:Bsupp}
\supp ( \Up{P} ) = \bigcup_{C \in \C{C}_{\Up{P}}} C.
\end{equation}
Conversely, for any subset $\C{C} \subset \quotient{\C{D}}{\sim_{\C{B}}}$ 
of components of $\C{D}$, 
any strict sub-probability measure $\xi \in \C{M}_+(D)$, 
such that $\supp(\xi) = D$, and any probability measure $\mu
\in \C{M}_+^1 \bigl( \quotient{\C{D}}{\sim_{\C{B}}}
\bigr)$, such that $\supp(\mu) = \C{C}$, 
the probability measure on $\C{D}$ defined as 
\begin{equation}
\label{eq:Bproc}
\Up{P}(s)  = \sum_{C \in \C{C}} \B{1}(s \in C) \mu(C) \B{P}
\bigl( \wt{S}_{\xi} = s \, | \, \wt{S}_{\xi} \in C \bigr), \qquad s \in \C{D},
\end{equation}
is a $\C{B}$-Markov substitute measure on the domain $\C{D}$ with 
support  
\begin{equation}
\label{eq:Bsupp2}
\supp(\Up{P}) = \bigcup_{C \in \C{C}} C. 
\end{equation}
\end{prop}
\begin{proof}
Let 
\[ 
\C{C}_{\Up{P}} = \bigl\{ C \in \quotient{\C{D}}{\sim_{\C{B}}} \, ; \, 
\Up{P}(C) > 0 \bigr\},
\]
so that obviously $\ds \supp(\Up{P}) \subset \bigcup_{C \in \C{C}_{\Up{P}}}
C$. 
Conversely, for any $C \in \C{C}_{\Up{P}}$, 
there is $s_C \in C$ such that $\Up{P}(s_C) > 0$. 
By definition of $\quotient{\C{D}}{\sim_{\C{B}}}$, for any $s \in C \in \C{C}_{\Up{P}} \in \quotient{\C{D}}{\sim_{\C{B}}}$, 
there are finite sequences 
\begin{align*}
s_k & \in C, & 0 & \leq k \leq \ell,\\  
x_k & \in (D^*)^2 , & 1 & \leq k \leq \ell, \\ 
B_k & \in \C{B}, & 1 & \leq k \leq \ell, \\
(y_k, y_k') & \in B_k^2, & 1 & \leq k \leq \ell, 
\end{align*}
such that $s_0 = s_C$, $s_{\ell} = s$, $s_{k-1} = \alpha(x_k, y_k')$ and 
$s_k = \alpha(x_k, y_k)$, $1 \leq k \leq \ell$. 
Consequently
\[
\Up{P}(s) = \Up{P}(s_C) \prod_{k=1}^\ell 
\exp \bigl( \beta ( y'_k, y_k ) \bigr)  
> 0, 
\]
so that $C \subset \supp(\Up{P})$. Therefore $\ds \bigcup_{C \in \C{C}_{\Up{P}}}
C \subset \supp ( \Up{P} )$, proving \cref{eq:Bsupp}.   

Let us now come to the second part of the proposition. 
The support of $\wt{S}_{\xi}$ being $D^+$, it is clear that 
the support of $S$ defined by \cref{eq:Bproc} 
is the one defined by \cref{eq:Bsupp2}.
For any $B \in \C{B}$, let us define 
\[ 
\beta (y,y') = \log \biggl( 
\frac{\B{P}( \wt{S}_{\xi} = y')}{\B{P}( \wt{S}_{\xi} = y)} 
\biggr), \qquad
y, y' \in B.
\] 
For any $B \in \C{B}$, any $x \in (D^*)^2$, any $y, y' \in B$, 
such that $\alpha(x, \{y,y'\}) \subset \C{D}$, 
there is $C \in \quotient{\C{D}}{\sim_{\C{B}}}$ such that $\alpha(x, \{y,y'\}) \subset C$ and  
\begin{align*}
\Up{P} \bigl( \alpha(x, y') \bigr) & 
= \frac{ \mu(C) \B{P} \bigl(
\wt{S}_{\xi} = \alpha(x, y') \bigr)}{ 
\B{P}\bigl( \wt{S}_{\xi} \in C \bigr)} \\
& = \frac{ \mu(C) \B{P} \bigl(
\wt{S}_{\xi} = \alpha(x, y) \bigr) \B{P} \bigl( \wt{S}_{\xi} = y' \bigr)}{ 
\B{P}\bigl( \wt{S}_{\xi} \in C \bigr) 
\B{P}\bigl( \wt{S}_{\xi} = y \bigr)} \\
& = \Up{P} \bigl( \alpha(x,y) \bigr) \exp \bigl( \beta (y,y') \bigr),
\end{align*}
proving that $B$ satisfies \vref{eq:swapsub}, and therefore that 
$\Up{P} \in \F{M}(\C{D}, \C{B})$.
Remark that in this proof, we have 
taken advantage of the fact that the Markov substitute 
property is stable by conditioning. 
\end{proof}

\begin{lemma}
\label{lem:equal}
In the case when $\F{M}(\C{D}, \C{B}) = \F{M}(\C{D}, \C{B}')$, 
then $\quotient{\C{D}}{\sim_{\C{B}}} = 
\quotient{\C{D}}{\sim_{\C{B}'}}$. 
\end{lemma}
\begin{proof}
Assume that the hypothesis stated in the lemma is satisfied. 
According to the previous proposition, for any $C \in \quotient{\C{D}}{\sim_{
\C{B}}}$, there is $\Up{P} \in \F{M}(\C{D}, \C{B})$ such that $\supp ( \Up{P} 
) = C$. Since $\Up{P} \in \F{M}(\C{D}, \C{B}')$, there is also $\C{C}' 
\subset \quotient{\C{D}}{\sim_{\C{B}'}}$ such that $C = \supp(\Up{P}) 
= \bigcup \C{C}'$. Therefore $\quotient{\C{D}}{\sim_{\C{B}}}$ is 
a coarser partition than $\quotient{\C{D}}{\sim_{\C{B}'}}$. 
As $\C{B}$ and $\C{B}'$ play symmetric roles, 
the reverse is also true, so that the two partitions are equal, 
each being coarser than the other one.  
\end{proof}

We are now going to show that the set $\F{M}(\C{D}, \C{B})$ 
of $\C{B}$-Markov processes on the domain $\C{D}$ forms 
an exponential family, although we will unfortunately 
do it in a non constructive way: we will not be able to 
provide an efficient algorithm to compute the corresponding 
energy function (or in other terms sufficient statistics). 
Moreover, we will restrict ourselves to the case when 
$\C{B}$ is a finite family of finite subsets of $D^+$. 

\begin{dfn}
Given a domain $\C{D} \subset D^+$, a set $\C{B}$ 
of subsets of $D^+$ and a subset $\C{C} \subset 
\quotient{\C{D}}{\sim_{\C{B}}}$ of components of 
$\quotient{\C{D}}{\sim_{\C{B}}}$,  
define $\F{M}_{\C{C}}\bigl( \C{D}, \C{B} \bigr)$ 
as the set of $\C{B}$-Markov probability measures on $\C{D}$ 
whose support is $\ds \bigcup \C{C} = \bigcup_{C \in \C{C}} C$.  
\end{dfn}

\begin{prop}
\label{prop:expMod}
Given any domain $\C{D} \subset D^+$---that may be infinite and may be 
$D^+$ itself---any finite set $\C{B}$ of finite subsets of $D^+$, 
there is a finite set of pairs $\C{P}$, that we can choose such that 
each one is included in a member of $\C{B}$, such that the sets 
of $\C{B}$-Markov and $\C{P}$-Markov substitute processes on the domain 
$\C{D}$ are the 
same---that is such that $\F{M}(\C{D}, \C{B}) = \F{M}(\C{D}, \C{P})$.
Moreover, since it is finite, we can if required choose $\C{P}$ to be minimal for the inclusion 
relation---meaning that removing a pair from $\C{P}$ would 
change $\F{M}(\C{D}, \C{P})$ to a broader model. 

For any finite subset of components 
$\C{C} \subset \quotient{\C{D}}{\sim_{\C{B}}}$, 
let us define the set of active pairs as 
\[ 
\C{A} = \bigl\{ \{y,y'\} \in \C{P}, \text{ there is } 
x \in (D^*)^2 \text{ such that } \alpha(x,\{y,y'\}) \subset \bigcup_{C \in  
\C{C}} C \bigr\}. 
\] 
There is a  
non empty subset of pairs $\C{F} 
\subset \C{A}$, a matrix $ \bigl( e_{i,j}, i \in \C{F}, 
j \in \C{A} \setminus \C{F} 
\bigr)$, a finite index set $\C{I} = \C{F} \cup \C{C}$, 
and energy functions 
\begin{align*}
U_i : \bigcup_{C \in \C{C}} C & \rightarrow \B{R}, \\  
s & \mapsto U_i(s), \qquad i \in \C{I}, 
\end{align*}
such that the set $\F{M}_{\C{C}}(\C{D}, \C{B})$ 
is the linear exponential family
\begin{align*}
\F{M}_{\C{C}} ( \C{D}, \C{B} ) & = 
\Biggl\{ \Up{P}_{\beta} \,  ; \,  \beta \in \F{B} \subset \B{R}^{\C{I}},
 \Up{P}_{\beta} (s) = Z_{\beta}^{-1} 
\exp \biggl( - \sum_{i \in \C{I}} 
\beta_i U_i(s) \biggr), s \in \bigcup \C{C} 
\Biggr\}, \\
\text{where } Z_{\beta} & = \sum_{s \in \bigcup \C{C}} \exp \biggl( - \sum_{i \in \C{I}} 
\beta_i U_i(s) \biggr), \\ 
\text{and } \F{B} & = \Bigl\{ \beta \in \B{R}^{\C{I}}, 
Z_{\beta} < \infty \Bigr\}. 
\end{align*}
Moreover, for any $\beta \in \F{B}$,  
the substitute exponent under $\Up{P}_{\beta}$ on $\C{D}$ of any pair $i = 
\{ y_{i,0}, y_{i,1} \} 
\in \C{F}$, indexed in a suitable way compatible with the definition 
of $U_i$, is given by 
\[ 
\beta (y_{i,0},y_{i,1}) = \beta_i,
\] 
whereas the substitute exponent of $j \in \C{A} \setminus \C{F}$ 
is given by
\[
\beta (y_{j,0},y_{j,1}) = \sum_{i \in \C{F}} 
\beta_i \, e_{i,j},
\]
and the substitute exponent of $j \in \C{P} \setminus \C{A}$ 
can be arbitrarily set to any real value. 
On the other hand, the probabilities $\bigl( \Up{P}_{\beta}(C), C \in \C{C} \bigr)$ 
are given by 
\[ 
\Up{P}_{\beta} (C) = Z_{\beta}^{-1} \exp(\beta_C) \sum_{s \in C} 
\exp \biggl( - \sum_{i \in \C{F}} \beta_i \, U_i(s) \biggr).
\] 
As a consequence, for any $\beta, \beta' \in \F{B}$, 
\[ 
\Up{P}_{\beta'} = \Up{P}_{\beta} \iff 
\begin{cases}
\beta_i' = \beta_i, & i \in \C{F}, \\ 
\beta_C' = \beta_C + \log \bigl( Z_{\beta'} / Z_{\beta} 
\bigr), & C \in \C{C}. 
\end{cases} 
\] 
\end{prop}
\begin{proof}
Let us start with the set of pairs 
\[ 
\C{P}' = \bigl\{ \{y , y'\} \subset B \in \C{B}, y \neq y' \bigr\}.
\] 
In view of the second statement of \vref{prop1.1}, 
the property that $\Up{P} \in \F{M}(\C{D}, \C{B})$ 
can be reformulated as 
\begin{multline*}
\Up{P}\bigl( \alpha(x, y') \bigr)
= \Up{P}\bigl( \alpha(x, y) \bigr) \exp \bigl( \beta(y,y') \bigr), 
\\ x \in (D^*)^2, (y,y') \in \bigcup_{B \in \C{B}} B^2, 
\alpha(x, \{y,y'\}) \subset \C{D},
\end{multline*}
for some global exponent function $\beta : \bigcup_{B \in \C{B}} 
B^2 \rightarrow \B{R}$. Since 
\[
\bigcup_{B \in \C{B}} B^2 = 
\bigcup_{B \in \C{P}'} B^2, 
\] 
this shows that $\F{M}(\C{D}, \C{B}) = \F{M}(\C{D}, \C{P}')$.

Since $\C{B}$ is assumed to be a finite family of finite sets, 
$\C{P}'$ is a finite set of pairs and 
we can obtain if wanted a minimal set of pairs $\C{P}$ 
by removing redundant pairs from $\C{P}'$.

Remark that according to \vref{lem:equal}, $\quotient{\C{D}}{\sim_{\C{B}}} 
= \quotient{\C{D}}{\sim_{\C{P}}}$.

For any $\C{C} \subset \quotient{\C{D}}{\sim_{\C{B}}}$, 
according to \vref{prop:2.2}, the set $\F{M}_{\C{C}}(\C{D}, \C{B})$ 
is non empty, 
and is equal to $\F{M}_{\C{C}}(\C{D}, \C{P})$, 
by construction of $\C{P}$. 

Let us index the corresponding set of active pairs $\C{A}$
defined in the proposition as 
\[ 
\C{A} = \bigl\{ (y_{i,0}, y_{i,1}), 1 \leq  i \leq I \bigr\}.
\] 

For any two states $s$ and $s' \in C \in \C{C}$, 
let us define the set of paths connecting $s$ to $s'$ as 
\[ 
\F{P}_{s,s'} = \Bigl\{ (x_j, i_j, \sigma_j), 0 \leq j \leq L, 
L \in \B{N}, x_j \in (D^*)^2, 1 \leq i_j \leq I, \sigma_j \in \{0,1\} \Bigr\},
\]  
such that 
\begin{align*}
\alpha(x_0, y_{i_0,\sigma_0} ) & = s, \\ 
\alpha(x_L, y_{i_L, 1 - \sigma_L} ) & = s', \\ 
\alpha(x_j, y_{i_j, \sigma_j}) & = \alpha(x_{j-1}, y_{i_{j-1}, 
1 - \sigma_{j-1}}) \in \C{D}, 
\qquad 0 < j \leq L.
\end{align*}
We see from the description of $\quotient{\C{D}}{\sim_{\C{P}}}$ 
as the connected components of the graph $ \bigl\{ 
\bigl(\alpha(x, y),  \alpha(x, y') \bigr), x \in (D^*)^2, \{y,y'\} \in \C{P} 
\bigr\} \cap \C{D}^2$ defined by $\C{P}$ 
that $\F{P}_{s,s'} \neq \varnothing$ and that in fact 
\[
\alpha(x_j, y_{i_j, \sigma_j}) = \alpha(x_{j-1}, y_{i_{j-1}, 
1 - \sigma_{j-1}}) \in C, \qquad 0 < j \leq L. 
\] 
For any $\pi = (x_j, i_j, \sigma_j)_{0 \leq j \leq L} \in \F{P}_{s,s'}$,  
let 
\[
U_i(\pi) = \sum_{j=0}^L \B{1} \bigl( i_j = i \bigr) \bigl( 2 \sigma_j - 1 
\bigr).
\]
Let us choose for any $C \in \C{C}$ a reference state $s_C$, and 
consider the set of loops $\ds \F{L} = \bigcup_{C \in \C{C}} \F{P}_{s_C, s_C}$. 
Consider the $I$ 
functions 
\begin{align}
U_i : \F{L} & \rightarrow \B{R}, \nonumber  \\ 
\pi & \mapsto U_i(\pi), \label{eq:Ufunc}
\end{align}
where $1 \leq i \leq I$. 
Reindexing the set of pairs if necessary, assume that $U_{k}$,  
$K < k \leq I$ is a maximum free subset of $\bigl( U_i,  1 \leq i \leq I 
\bigr)$, so that there is a matrix $e_{i,j}$, such that 
\begin{equation}
\label{eq:Udep}
U_i = - \sum_{k=K+1}^I e_{i,k} U_k, \qquad 1 \leq i \leq K.
\end{equation}
(In the case when $U_i$, $1 \leq i \leq I$ are already 
linearly independent, it should be understood that $K = 0$ 
and that the above statement is void, since no index $i$ satisfies 
$1 \leq i \leq 0$. Let us also remark that we can have $K = I$ in the case 
when all the functions $U_i$ are equal to the null function.)
For any set of parameters $\beta = (\beta_i, 1 \leq i \leq I) \in \B{R}^I$, 
for any path $\pi = (x_j, i_j, \sigma_j)_{0 \leq j \leq L} \in \F{P}_{s,s'}$,  
let 
\[
w_{\beta} ( \pi ) = \sum_{j=0}^L  (1 - 2 \sigma_j) \beta_{i_j} 
= - \sum_{i=1}^I \beta_i U_i(\pi).
\]
\begin{lemma}
For any $\beta \in \B{R}^I$ 
such that 
\begin{equation}
\label{eq:2.6}
\beta_k = \sum_{i=1}^K \beta_i \, e_{i,k}, \qquad K < k \leq I,
\end{equation}
with the convention that in the case when $K = 0$, 
$\beta_k = 0$, $0 < k \leq I$, 
and in the case when $K = I$, the assumption is void, 
for any $C \in \C{C}$, 
any $s \in C$, and any $\pi, \pi' \in \F{P}_{s_C, s}$,
\[ 
w_{\beta} (\pi') = w_{\beta} (\pi).
\] 
\end{lemma}
\begin{proof}
There is $\pi'' \in \F{P}_{s_C, s_C}$ such that $w_{\beta} (\pi'')  
= w_{\beta}(\pi') - w_{\beta} (\pi)$ (where $\pi''$ is built in an obvious way as 
the concatenation of $\pi'$ and of the reverse of $\pi$). 
\Vref{eq:2.6} ensures that $w_{\beta}(\pi'') = 0$.  Indeed in this case
\begin{multline*}
w_{\beta}(\pi'') = - \sum_{i=1}^I \beta_i U_i(\pi'') = \sum_{i=1}^K \sum_{k=K+1}^I \beta_i 
e_{i,k} U_k (\pi'') - \sum_{k=K+1}^I \beta_k U_k(\pi'') 
\\ = \sum_{k=K+1}^I \Biggl( \sum_{i=1}^K \beta_i \, e_{i,k} - \beta_k \Biggr) 
U_k(\pi'') = 0.
\end{multline*}
\end{proof}
For any $s \in C \in \C{C}$, let us choose $\pi_s \in \F{P}_{s_C, s}$,
and define 
\begin{equation}
\label{eq:3.7.2}
U_i(s) = U_i(\pi_s) + \sum_{k=K+1}^I e_{i,k} \, U_k ( \pi_s ), \qquad 
1 \leq i \leq K, 
\end{equation}
so that, when \cref{eq:2.6} holds,  
\[
w_{\beta} (\pi_s) = \- \sum_{i=1}^K \beta_i \, U_i(s).
\]
Let us index the set of components $\C{C}$ (assumed to be finite) 
as 
\[ 
\C{C} = \{ C_{K+1}, \dots, C_{J} \}, 
\] 
and define 
\[
U_j (s) = - \B{1} ( s \in C_j ), \qquad K < j \leq J. 
\]
Define the parameter set $\F{B} \subset \B{R}^J$ as 
\[ 
\F{B} = \biggl\{ \beta \in \B{R}^J, Z_{\beta} \overset{\text{\rm def}}{=} 
\sum_{s \in \bigcup \C{C}} 
\exp \biggl( - \sum_{j=1}^J \beta_j U_j(s) \biggr) < \infty \biggr\}. 
\] 
For each $\beta \in \F{B}$, define $\Up{P}_{\beta} \in \C{M}_+^1 \bigl( 
\bigcup \C{C} \bigr)$ as 
\[
\Up{P}_{\beta}(s) = Z_{\beta}^{-1} \exp \biggl( - \sum_{j=1}^J 
\beta_j U_j(s) \biggr), \qquad s \in \bigcup_{C \in \C{C}} C.  
\]

\begin{lemma}
For any $\beta \in \F{B}$, $\Up{P}_{\beta} \in \F{M}_{\C{C}}(\C{D}, \C{P})$. 
Define
\[
\beta_k' = 
\begin{cases}
\ds \beta_k, & 1 \leq k \leq K, \\ 
\ds \sum_{i=1}^K \beta_i \, e_{i,k}, & K < k \leq I.
\end{cases}
\]
For any $i \in \{1, \dots, I\}$,  
\[
\beta (y_{i,0}, y_{i,1}) = \beta_i' 
\]
is the substitute exponent of $\{y_{i,0}, y_{i,1}\}$ for the Markov 
substitute measure $\Up{P}_{\beta}$ on $\C{D}$. When $\{y, y'\} \in \C{P} 
\setminus \C{A}$ is not an active pair, 
its substitute exponent $\beta (y,y')$ is not uniquely 
defined and can be set to any arbitrary real value. 
\end{lemma}

\begin{proof}
First of all, it is immediate to see that $\supp ( \Up{P}_{\beta} ) = 
\bigcup \C{C}$. 
Consider first any pair $\{y, y'\} \in \C{P} \setminus \C{A}$ and 
any $x \in (D^*)^2$ such that $\alpha(x, \{y, y'\}) \subset \C{D}$. 
Since $\alpha(x, \{y, y'\}) \not \subset \bigcup \C{C}$, 
and since $\alpha(x, y) \sim_{\C{P}} \alpha(x,y')$, 
necessarily 
\[
\alpha \bigl( x, \{y, y'\} \bigr) \cap \bigl( \bigcup \C{C} \bigr) = \varnothing, 
\]
so that $\Up{P}_{\beta}\bigl( \alpha(x, y) \bigr) = \Up{P}_{\beta}
\bigl( \alpha(x, y') \bigr) = 0$ and  
\[
0 = \Up{P}_{\beta} \bigl( \alpha(x, y') \bigr) = \Up{P}_{\beta}
\bigl( \alpha(x, y) \bigr) \exp \bigl( \beta (y,y') \bigr),
\]
for any choice of $\beta (y,y') \in \B{R}$. 
Consider any $i \in \{ 1, \dots,  I \}$ and any $x \in ( D^* )^2$, 
such that $z = \alpha(x, y_{i,0}) \in \C{D}$
and $z' = \alpha(x, y_{i,1}) \in \C{D}$. 
Since $z \sim_{\C{P}} z'$, there is $C \in \quotient{\C{D}}{\sim_{\C{P}}}$ 
such that $\{z, z'\} \subset C$. 
If $C \notin \C{C}$, then $\Up{P}_{\beta}(C) = 
\Up{P}_{\beta}(z) = \Up{P}_{\beta}(z') = 0$, so that 
obviously
\[
0 = \Up{P}_{\beta}(z') = \Up{P}_{\beta} (z) \exp( \beta_i' ). 
\]
Assume now on the other hand that $C \in \C{C}$. 
Let $\pi'$ be the 
concatenation of $\pi_z$ (defined before \vref{eq:3.7.2}) 
and $(x, i, 0)$. We see that  
$\pi' \in \F{P}_{s_C, z'}$, so that  
\[
- \sum_{j=1}^K \beta_j U_j(z') = 
w_{\beta'} (\pi_{z'}) = w_{\beta'} (\pi') = w_{\beta'} (\pi_{z}) 
+ \beta_i' 
= - \sum_{j=1}^K \beta_j U_j(z)  
+ \beta_i'.
\]
Remark also that $U_j(z) = U_j(z')$, $K < j \leq J$, since 
$z$ and $z'$ belong to the same component $C \in \C{C}$. 
Therefore
\begin{multline*}
\log \bigl( \Up{P}(z') \bigr) + \log ( Z_{\beta} ) 
= - \sum_{j=1}^J \beta_j U_j(z') 
\\ = - \sum_{j=1}^J \beta_j U_j(z) + \beta_i' 
= \log \bigl( \Up{P}_{\beta}(z) \bigr)
+ \log ( Z_{\beta} ) + \beta'_i, 
\end{multline*}
proving that 
\[
\Up{P}_{\beta}(z') = 
\Up{P}_{\beta}(z) \exp ( \beta_i' ), 
\]
and therefore that $\Up{P}_{\beta} \in \F{M}_{\C{C}}(\C{D}, \C{P})$
with the prescribed substitute exponents. 
\end{proof}
\begin{lemma}
Let us put 
\[ 
Z_{\beta_{1:K},j} = \sum_{s \in C_j} \exp \biggl( - \sum_{i=1}^K \beta_i U_i(s) 
\biggr), 
\] 
where $\beta_{1:K} = \{ \beta_1, \dots, \beta_K \}$, 
so that 
\[ 
Z_{\beta} = \sum_{j=K+1}^J \exp(\beta_j) Z_{\beta_{1:K},j}.
\] 
The parameters $\beta_{(K+1):J} = \{ \beta_{K+1}, \dots, \beta_J \}$ are 
related to $\Up{P}_{\beta}$ by the following relation
\[ 
\Up{P}_{\beta}(C_j) = \exp ( \beta_j ) \frac{ Z_{\beta_{1:K}, j}}{Z_{\beta}},
\qquad K < j \leq J.
\] 
\end{lemma}
\begin{proof}
By definition of $Z_{\beta}$, 
\[ 
Z_{\beta} = \sum_{s \in \bigcup \C{C}} 
\exp \biggl( - \sum_{i=1}^J \beta_i U_i(s) \biggr) 
= \sum_{j=K+1}^J \sum_{s \in C_j}  
\exp \biggl( - \sum_{i=1}^J \beta_i U_i(s) \biggr). 
\] 
Since $U_j(s) = - \B{1}(s \in C_j)$, $K < j \leq J$, 
\[
Z_{\beta} = \sum_{j=K+1}^J \sum_{s \in C_j} 
\exp \biggl( \beta_j - \sum_{i=1}^K \beta_i U_i(s) 
\biggr) = \sum_{j=K+1}^J \exp(\beta_j) Z_{\beta_{1:K}, j}. 
\]
Moreover, the definition 
of $\Up{P}_{\beta}$ implies that 
\begin{multline*}
\Up{P}_{\beta}(C_j) = \sum_{s \in C_j} Z_{\beta}^{-1} \exp \biggl( 
- \sum_{i=1}^J \beta_i U_i(s) \biggr) \\ 
= \sum_{s \in C_j} Z_{\beta}^{-1} \exp \biggl( 
\beta_j - \sum_{i=1}^K \beta_i U_i(s) \biggr)  
= \exp(\beta_j) \frac{Z_{\beta_{1:K},j}}{Z_{\beta}}, \qquad K < j \leq J.
\end{multline*}
\end{proof}
\begin{lemma}
For any $\beta, \beta' \in \F{B}$ 
\[
\Up{P}_{\beta'} = \Up{P}_{\beta} \iff 
\begin{cases} 
\beta_i' = \beta_i, & 1 \leq i \leq K, \\ 
\beta_i' = \beta_i + \log \bigl( Z_{\beta'} / Z_{\beta} \bigr), & K < i \leq J.
\end{cases}
\]
\end{lemma}
\begin{proof}
This is a consequence of the two previous lemmas and the fact that 
$\Up{P} \in \F{M}_{\C{C}}(\C{D}, \C{P})$ is determined by the substitute 
exponents of active Markov substitute pairs and the 
probabilities 
\[
\Up{P}(C_j) = \exp(\beta_j) \frac{Z_{\beta_{1:K, j}}}{Z_{\beta}}, \qquad K < j \leq J,
\]
that are themselves 
determined by $\Up{P}$. The constant $\log \bigl( Z_{\beta'} / Z_{\beta} 
\bigr)$ is there just 
because we chose not to break the symmetry of the role played by 
the components $C_j$, $K < j \leq J$. Since $\sum_{j=K+1}^J 
\Up{P}(C_j) = 1$, we could have characterized the vector 
of probabilities $(\Up{P}(C_j))_{K < j \leq J}$ by only 
$J - K - 1$ paramters instead of $J-K$. In other words we could 
have used the fact that $U_{J}(s) = - 1 - \sum_{j=K+1}^{J-1} U_{j}(s)$, 
to remove the parameter $\beta_J$ from the representation. 
\end{proof}
\begin{lemma}
The reverse inclusion
\[
\F{M}_{\C{C}}(\C{D}, \C{B}) \subset \bigl\{ \Up{P}_{\beta}, \beta \in \F{B} \bigr\}
\]
is also satisfied.
\end{lemma}
\begin{proof}
Consider any probability measure $\Up{P} 
\in \F{M}_{\C{C}}(\C{D}, \C{B}) = \F{M}_{\C{C}}(\C{D}, \C{P})$ 
(we know from \vref{prop:2.2} that $\F{M}_{\C{C}}(\C{D}, \C{P} ) \neq \varnothing$). 
Let 
\[ 
\beta_i' = \beta (y_{i,0}, y_{i,1}), \qquad 1 \leq i 
\leq I,
\] 
where $\beta (y_{i,0}, y_{i,1})$ 
is the substitute exponents of $\{y_{i,0}, y_{i,1}\}$ under 
$\Up{P}$ on $\C{D}$. 
From the Markov substitute property expressed by \vref{eq:swapsub}, 
we see that for any $s, s' \in C \in \C{C}$, any $\pi \in \F{P}_{s,s'}$, 
\[
\Up{P}(s') = \Up{P}(s) \exp \bigl( w_{\beta'}(\pi) \bigr) . 
\]
Therefore, for any $s \in C \in \C{C}$, 
\[ 
\Up{P}(s) = \Up{P}(s_C) \exp \bigl( w_{\beta'}(\pi_s) \bigr) = 
\exp \biggl( - \sum_{i=1}^I \beta_i' U_i(\pi_s) \biggr) \Up{P}(s_C)
\]  
and for any $\pi \in \F{L}$, 
\[ 
1 = \exp \bigl( w_{\beta'}(\pi) \bigr) = 
\exp \biggl( - \sum_{i=1}^I \beta_i' U_i (\pi) \biggr).  
\] 
Using \vref{eq:Udep}, we obtain that 
\[ 
\sum_{k=K+1}^I \biggl( \beta_k' - \sum_{i=1}^K \beta_i' e_{i,k} \biggr)  
U_k(\pi) = 0, \qquad \pi \in \F{L}, 
\] 
so that 
\[
\beta_k' = \sum_{i=1}^{K} \beta_i' \, e_{i,k}, \qquad K < k \leq I,
\]
since the functions $(U_k, K < k \leq I)$ defined by 
\vref{eq:Ufunc} are linearly independent. 
Therefore 
\[
\sum_{i=1}^I \beta_i' \, U_i(\pi_s) = \sum_{i=1}^K \beta_i' \, U_i(s), 
\qquad s \in \bigcup \C{C}. 
\]
On the other hand, putting
\[ 
\beta_j = \log \bigl( \Up{P}(s_{C_j}) \bigr), \qquad K < j \leq J, 
\]
we obtain that
\[
\Up{P}(s_C) = \exp \biggl( - \sum_{j=K+1}^J \beta_j \, U_j(s) \biggr), \qquad
\text{ for any } C \in \C{C}. 
\]
Let us define 
\[ 
\beta_i = \beta_i', \qquad 1 \leq i \leq K.
\] 
For any $s \in C \in \C{C}$, we obtain that 
\begin{multline*}
\Up{P}(s) = 
\exp \biggl( - \sum_{i=1}^I \beta_i' \, U_i(\pi_s) \biggr) \Up{P}(s_C) 
\\ = \exp \biggl( - \sum_{i=1}^K \beta_i' \, U_i(s) - \sum_{j=K+1}^J \beta_j 
\, U_j(s) \biggr)  
=  \exp \biggl( - \sum_{j=1}^J \beta_j \, U_j(s) \biggr).
\end{multline*}
Remark that with this choice of parameter $\beta$,
\[
Z_{\beta} = \sum_{s \in \bigcup \C{C} } \exp \biggl( - \sum_{j=1}^J \beta_j U_j(s) \biggr)  =  \sum_{s \in \bigcup \C{C}} \Up{P}(s) = 1, 
\]
so that $\Up{P} = \Up{P}_{\beta}$. 
\end{proof}
We obtain \vref{prop:expMod} by gathering the previous lemmas together
and indexing directly $\C{P}$ and $\C{C}$ by themselves instead of 
using numerical indices as in the proofs.  
\end{proof}

Remark that \cref{prop:expMod} implies that the maximum likelihood estimator 
in $\F{M}_{\C{C}}(\C{D}, \C{B})$ is an asymptotically efficient estimator of the 
parameters $\beta_i$, $i \in \C{I}$. This property does not provide 
a practical estimator though, since the construction of the energy 
functions $U_i$, $i \in \C{I}$ is not explicit, at least in the general 
case. It is nevertheless possible to approximate the maximum 
likehood estimator without computing the energy functions 
explicitly, using a Monte-Carlo simulation algorithm that 
is proposed in \cite[page 135]{Mainguy} and will 
be further presented and analysed in another publication. 

\section{Examples} 

\subsection{Some simple recursive structures} 

Let us give an example showing that a minimal set of pairs is 
not necessarily free, so that we can have $\C{P} \neq \C{F}$, 
and that the set of free pairs $\C{F}$---and therefore the form 
of the energy---may depend on the support   
$\ds \bigcup_{C \in \C{C}} C$ of the Markov substitute process 
in a non trivial way.

On the three words dictionary $D = \{a, b, c\}$, consider the domain 
$\C{D} = D^+$ and the 
family of subsets $\C{B} = \{ B_1, B_2 \}$, 
where $B_1 = \{ab, a \}$ and $B_2 = \{bc, c\}$. 
Define
\begin{align*}
C_1 & = \{ a b^n c, \; n \in \B{N} \}, \\  
C_2 & = \{ b^m c a b^n,  \; m, n \in \B{N} \},  \\ 
C_3 & = \{ b^k c a b^m c a b^n,  \; k, m, n \in \B{N} \}. 
\end{align*}
It is easy to check that $C_k \in \quotient{D^+}{\sim_{\C{B}}}$, 
$1 \leq k \leq 3$, so that we may consider the three 
supports $C_k$, corresponding to $\C{C}_k = \{ C_k \} 
\subset \quotient{D^+}{\sim_{\C{B}}}$. 

In $C_1$, we have the loop 
\[ 
ac = \alpha\bigl( (\varepsilon, c), a \bigr) \rightarrow 
\alpha \bigl( (\varepsilon, c), ab \bigr) = abc = 
\alpha \bigl( (a, \varepsilon), bc \bigr) \rightarrow 
\alpha\bigl( (a, \varepsilon), c \bigr) = ac,  
\] 
inducing the constraint
\[
\beta(a,ab) = \beta(c, bc), 
\]
so that 
\[ 
\F{M}_{\C{C}_1}(D^+, \C{B}) = \bigl\{ \Up{P} \in \C{M}_+^1(C_1) \, ; \, 
\Up{P}(ab^nc) = r(1 - r)^n, r \in ]0,1[ \bigr\}. 
\] 
In this case we can choose the set of free pairs either equal to 
$\C{F}(\C{C}_1)  = \{ B_1 \}$ 
or equal to $\C{F}(\C{C}_1)  = \{ B_2 \}$. 

In $C_2$, there is no non trivial loop, so that the set of free 
pairs is $\C{F}(\C{C}_2) = \C{B}$, 
and we get an exponential family with two parameters 
\begin{multline*}
\F{M}_{\C{C}_2} ( D^+, \C{B}) = \bigl\{ \Up{P} \in \C{M}_+^1(C_2) \, ; 
\\ \, \Up{P}(b^m c a b^n) = r_1 r_2(1 - r_1)^m (1 - r_2)^n, r_1, r_2 
\in ]0,1[ \bigr\}.
\end{multline*}

In $C_3$, we have the same non trivial loop as in $C_1$, 
imposing the constraint $\beta(a, ab) = \beta(c, bc)$, 
so that
we can choose $\C{F}(\C{C}_3) = \{ B_1 \}$ or 
$\C{F}(\C{C}_3) = \{ B_2 \}$, and  
\begin{multline*}
\F{M}_{\C{C}_3}(D^+, \C{B}) = \bigl\{ \Up{P} \in \C{M}_+^1(C_3) \, ; \, 
\Up{P}(b^kcab^mcab^n) = r(1-r)^{k+m+n}, \\ k, m, n \in \B{N}, r \in ]0, 1[ 
\bigr\}.
\end{multline*}
Note also that $\F{M}_{\C{C}_1}(D^+, \C{B}) = \F{M}_{\C{C}_1} ( D^+, \{B_1\})
= \F{M}_{\C{C}_1} (D^+, \{B_2\})$, 
so that $\C{B}$ is not a minimal set of pairs on $\C{C}_1$, but 
is a minimal set of pairs both on $\C{C}_2$ and $\C{C}_3$.
Therefore, in $\F{M}_{\C{C}_3} \bigl( D^+, \C{B} \bigr)$, 
$\C{B}$ is a minimal set of pairs but is not a free set of pairs.  

\subsection{Links with Markov chains} 

Let us now consider another example, to show that Markov substitute 
processes contain Markov chains. Consider a finite state space $D$ 
and the family of subsitute sets 
\begin{align*}
\C{B} & = \bigl\{ \alpha \bigl( (a, b), D \bigr), (a, b) \in D^2 \bigr\}. 
\end{align*} 
\begin{lemma}
The model $\F{M}(D^L, \C{B})$ contains the law of 
all time homogeneous Markov chains $(S_1, \dots, S_L)$ 
with positive transition matrix $M$, that is the law of all Markov 
chains $(S_1, \dots, S_L)$ such that 
\[ 
\B{P}(S_{k} = b \, | \, S_{k-1} = a ) 
= M(a,b) > 0, \qquad (a, b) \in D^2, 1 < k \leq L.
\] 
\end{lemma}
\begin{proof}
If $S$ is such a Markov chain, 
for any context $x \in (D^*)^2$
and any 
\[ 
(a,y,b), (a,y',b) \in \alpha \bigl( (a,b), D \bigr) \in \C{B}, 
\] 
it is easy to check that 
\[ 
\B{P} \bigl[ S = \alpha \bigl( x, (a,y',b) \bigr) \bigr] > 0
\text{ if and only if }
\B{P} \bigl[ S = \alpha \bigl( x, (a,y,b) \bigr) \bigr] > 0. 
\] 
In this case, we can write $x$ as $x = (w_{1:k}, w_{k+4:L})$, where $0 \leq k \leq L-3$, 
$w_{1:k} = (w_1, \dots, w_k) \in D^k$, and $w_{k+4:L} 
= (w_{k+4}, \dots, w_L) \in D^{L-k-3}$, with the convention that  
$w_{1:0} = w_{L+1:L} = \varepsilon$. We can then compute 
\begin{multline*}
\B{P} \bigl[ S = \alpha \bigl( x, (a,z,b) \bigr) \bigr] \\ 
=  \B{P} \bigl[ S_{1:k+1} = \gamma(w_{1:k},a)  \bigr] 
M(a,z) M(z,b) \\ \times \B{P} \bigl[ S_{k+4:L} = w_{k+4:L} \, | \, 
S_{k+3} = b\bigr]. 
\end{multline*}
Therefore
\[
\B{P} \bigl[ S = \alpha \bigl( x, (a,y',b) \bigr) \bigr]  
 = \B{P} \bigl[ S = \alpha \bigl( x, (a,y,b) \bigr) \bigr]  
\exp \bigl[ \beta \bigl( (a,y,b), (a,y',b) \bigr) \bigr],  
\]
where 
\[ 
\beta \bigl( (a,y,b), (a,y',b) \bigr)  =
\log \Biggl( \frac{ M(a,y') M(y',b)}{M(a,y) M(y,b)} \Biggr),  
\] 
proving that $\B{P}_{S} \in 
\F{M} \bigl( D^L, \C{B} \bigr)$. 
\end{proof}
\begin{lemma}
\[ 
\quotient{D^L}{\sim_{\C{B}}} = \bigl\{ \alpha \bigl( (a, b), D^{L-2} \bigr), 
(a,b) \in D^2 \bigr\}. 
\] 
\end{lemma} 
\begin{proof}
It is easy to see that $\alpha \bigl( (a, b), D^{L-2} \bigr)$ 
is connected by the graph $\C{G}(D^L, \C{B})$ (this notation 
was defined in \vref{eq:3.1}). On the other hand, 
if $(w_{1:L}, w'_{1:L}) \in \C{G}(D^L, \C{B})$, then $w_1 = w'_1$
and $w_{L} = w'_{L}$, showing that there is no connection 
between $\alpha \bigl( (a, b), D^{L-2} \bigr)$ and its 
complement in $D^{L}$. 
\end{proof}
\begin{lemma}
For any random process $S \in D^L$ such that $\B{P}_S \in \F{M}(D^L, \C{B})$, there is a time-homogeneous Markov 
chain $(X_1, \dots, X_L)$ with positive transition matrix $M$, 
such that for any $(a, b) \in D^2$, such that $
\B{P} ( S_1 = a, S_L = b ) > 0$, 
\[
\B{P}_{\ds S | S_1 = a, S_L = b} = \B{P}_{\ds X | X_1 = a, X_L = b}, 
\]
whereas the marginal law of the pair of end points 
$\B{P}_{\ds (S_1, S_L)}$ may be any arbitrary probability measure. 
On the other hand, the set of possible 
values of the probability measure 
$\B{P}_{\ds X_1, X_L}$ is constrained by the relation 
\[
\B{P} \bigl( X_L = b \, | \, X_1 = a \bigr) = M^{L-1}(a,b). 
\] 
\end{lemma}
\begin{proof}
In fact one can see from the definition of Markov substitute sets, that
for this specific choice of $\C{B}$, the model $\F{M}(D^L,\C{B})$ 
is the set of one dimensional random fields with prescribed boundary 
conditions. Building on this remark, a slight modification of the 
proof that a stationary one dimensional random field with a finite 
state space is a stationary Markov chain gives the result stated 
in the lemma. For the sake of completeness, we give here a proof 
adapted from \cite[page 45]{Georgii}.   

Let $S$ be as described in the lemma, and $\beta$ its substitute exponent.
Let $c$ in $D$ be some word of the dictionary. From the existence of the 
loop 
\begin{multline*}
cac^3 = \alpha 
\bigl( ( c, c), acc \bigr) \rightarrow cayc^2 = \alpha \bigl( (ca,\varepsilon), 
ycc \bigr) \\ \rightarrow 
caybc = \alpha \bigl( (c,c), ayb \bigr) \rightarrow 
cay'bc = \alpha \bigl( (ca,\varepsilon), y'bc \bigr) \\ \rightarrow 
cay'cc = \alpha \bigl( (c,c), ay'c \bigr) \rightarrow 
cac^3, 
\end{multline*}
we deduce that 
\[
\beta( ac^2, ayc ) + \beta(yc^2, ybc) + \beta( ayb, ay'b ) 
 - \beta( y'c^2, y'bc) - \beta( ac^2, ay'c) = 0,
\]
so that, for any $(a,y,y',b) \in D^4$, 
\begin{equation}
\label{eq:3.7}
\beta ( ayb, ay'b ) =  
 \beta( ac^2, ay'c)
+ \beta( y'c^2, y'bc) 
- \beta( ac^2, ayc ) - \beta(yc^2, ybc). 
\end{equation}
This means that all the substitute exponents are determined by the 
subfamily of substitute exponents consisting in 
\[ 
\beta( ac^2, abc), \qquad (a,b) \in D^2. 
\] 
Consider the positive matrix 
\[ 
A(a,b) = \exp \bigl[ \beta(ac^2, abc) \bigr], \qquad (a,b) \in D^2. 
\] 
According to the Perron-Frobenius theorem, it has a unique positive 
eigenvector $\bigl( \psi(a) > 0, a \in D \bigr)$ 
of norm $ \lVert \psi \rVert = 1$ associated to a real 
positive eigenvalue $\lambda$ (which turns out to be its spectral 
radius). 
Introduce the matrix 
\[
M(a,b) = \lambda^{-1} \psi(a)^{-1} A(a,b) \psi(b). 
\] 
and remark that it is a positive Markov transition 
matrix, due to the fact that $\psi$ is a positive eigenvector
with eigenvalue $\lambda$. 
We see from \cref{eq:3.7} that 
\[ 
\beta( ayb, ay'b) = \log \biggl( \frac{M(a,y') M(y', b)}{ 
M(a,y) M(y,b)} \biggr), \qquad (a,b,y,y') \in D^4. 
\] 

Let $(X_1, \dots, X_L)$ be a Markov chain with transition 
matrix $M$ and initial distribution some (arbitrary) 
probability measure on $D$ with full support. The previous equation 
shows that $\B{P}_{X} \in \F{M}(D^L. \C{B})$, with the same 
substitute exponents as $\B{P}_S$. From our study of Markov substitute 
processes, we know that the substitute exponents define the 
conditional distribution of a Markov substitute process 
on each component of $\quotient{D^L}{\sim_{\C{B}}}$ 
of positive probability, these components being 
described in the previous lemma. We conclude that for any $(a,b) \in D^2$ 
such that $\B{P} \bigl[ S \in \alpha \bigl( (a,b), D^{L-2} 
\bigr) \bigr]  > 0$, or in other words such that 
$\B{P} ( S_1 = a, S_L = b ) > 0$, 
\[ 
\B{P}_{\ds S | S \in \alpha \bigl( (a,b), D^{L-2} \bigr)} 
= \B{P}_{\ds X | X \in \alpha \bigl( (a,b), D^{L-2} \bigr)}, 
\] 
that can also be written as 
\[ 
\B{P}_{\ds S | S_1 = a, S_L = b} 
= \B{P}_{\ds X | X_1 = a, X_L = b}. 
\] 
\end{proof}

This study shows that the substitute sets of 
Markov chains have a very specific structure, and therefore 
that Markov substitute processes form a much richer 
family of models than Markov chains, while they can still be parametrized 
as exponential families, ensuring that they have some valuable  
properties as a statistical model. 

\section{Multi-dimensional extensions} 

  In this paper we have defined one-dimensional Markov substitute processes
and shown that they are an extension of one-dimensional Markov random 
fields. We can also generalize the notion of a multi-dimensional Markov random 
field by proposing a definition for multi-dimensional Markov 
substitute processes. 

  Let us define first a notion of Markov substitute process indexed by 
an arbitrary finite set $I$.

\begin{dfn}
Given a random process $S \in D^I$ indexed by a 
finite set $I$, we 
will say that $(J, B)$, where $J \subset I$ and $B \subset D^{J}$ is a Markov substitute set when there exists for any $y, y' \in B$, 
any $x \in D^{I \setminus J}$ a skew-symmetric substitute exponent 
$\beta_J (y,y') \in \B{R}$ 
such that 
\[ 
\B{P} \bigl( S_{J} = y' , S_{I \setminus J} = x \bigr) 
= 
\B{P} \bigl( S_{J} = y , S_{I \setminus J} = x \bigr) 
\exp \bigl[ \beta_J ( y,y') \bigr].  
\]  
\end{dfn}  

When $I \subset \B{Z}^d$ is part of a $d$-dimensional lattice,
one can make the definition translation invariant by imposing that 
for any $J \subset I$, $(J,B)$ is a translation invariant 
Markov substitute set 
if and only if, for any $t \in \B{Z}^d$ such that $t + J \subset I$, 
$(t+J, B \circ \tau_t)$ is a Markov substitute set in the sense 
of the above definition, where $\tau_t(i) = i - t$, and 
the substitute exponents are the same in the sense that
\[ 
\beta_{t + J}(y \circ \tau_t,y' \circ \tau_t ) = \beta_J(y,y'). 
\] 

It is easy to see that an obvious reformulation of 
\vref{prop:expMod} remains true for these two variants of 
the definition of Markov substitute sets. 
(We could also formulate analogous definitions for a 
process defined on a restricted domain $\C{D} \subset D^I$.) 

Remark however that these multi-dimensional variants are not properly speaking 
extensions of the one-dimensional setting, since in the one dimensional 
case, we can let the Markov substitute sets contain 
expressions of varying lengths, leading to the modeling 
of recursive structures. This ability of modeling recursive 
structures gives a special interest to 
one-dimensional Markov substitute processes. 
 
\section{Conclusion}

We presented here a slightly more general definition of Markov 
substitute processes than in \cite{Mainguy}, where it is assumed 
most of the time that $\C{D} = D^+$. We showed the main
property of the model, namely that it is for each legitimate choice 
of support an exponential family. One can show a host of interesting 
additional properties, some depending on further 
assumptions on the domain $\C{D}$. The model, noticeably,    
can be viewed as an extension of Markov chains and 
has deep connections with context free grammars 
\cite{Chomsky56,Chi98,Chi99}. One can also propose
algorithms to select Markov substitute models, estimate their 
parameters, simulate from them or compute the probability of 
a given sentence $s \in D^+$. We refer to \cite{Mainguy} and 
to forthcoming publications for more information and insight
on Markov substitute processes, a model at the crossroad of statistics and 
formal grammars.

\bibliographystyle{apalike}
\bibliography{biblio}

\begin{thebibliography}{}

\bibitem[Chi, 1999]{Chi99}
Chi, Z. (1999).
\newblock Statistical properties of probabilistic context-free grammars.
\newblock {\em Computational Linguistics}, 25(1):131--160.

\bibitem[Chi and Geman, 1998]{Chi98}
Chi, Z. and Geman, S. (1998).
\newblock Estimation of probabilistic context-free grammars.
\newblock {\em Computational linguistics}, 24(2):299--305.

\bibitem[Chomsky, 1956]{Chomsky56}
Chomsky, N. (1956).
\newblock Three models for the description of language.
\newblock {\em Information Theory, IRE Transactions on}, 2(3):113--124.

\bibitem[Georgii, 1988]{Georgii}
Georgii, H. (1988).
\newblock {\em Gibbs measures and phase transitions}.
\newblock Walter de Gruyter \& Co, Berlin.

\bibitem[Mainguy, 2014]{Mainguy}
Mainguy, T. (2014).
\newblock {\em Markov Substitute Processes}.
\newblock PhD thesis, ENS, 45 rue d'Ulm, 75005 Paris
  \url{http://www.normalesup.org/~mainguy/these.html}.

\bibitem[Roark, 2001]{Roark01}
Roark, B. (2001).
\newblock Probabilistic top-down parsing and language modeling.
\newblock {\em Computational linguistics}, 27(2):249--276.

\bibitem[Stabler, 2009]{Stabler09}
Stabler, E.~P. (2009).
\newblock Mathematics of language learning.
\newblock {\em Histoire, Épistémologie, Langage}, 31(1):127--145.

\end{thebibliography}

\end{document}